\documentclass{article} % For LaTeX2e
\usepackage{tccml_iclr2025_conference,times}
\usepackage{amsmath,amsfonts,bm}

\def\eqref#1{equation~\ref{#1}}
\def\1{\bm{1}}

\DeclareMathAlphabet{\mathsfit}{\encodingdefault}{\sfdefault}{m}{sl}
\SetMathAlphabet{\mathsfit}{bold}{\encodingdefault}{\sfdefault}{bx}{n}

\usepackage{hyperref}
\usepackage{url}
\usepackage{graphicx}
\usepackage{mathtools}
\usepackage{amssymb}
\usepackage{amsthm}
\usepackage{mathtools}
\usepackage{algorithm}
\usepackage{algorithmic}

\theoremstyle{plain}
\newtheorem{theorem}{Theorem}[section]
\newtheorem{lemma}[theorem]{Lemma}
\newtheorem{corollary}[theorem]{Corollary}

\theoremstyle{definition}
\newtheorem{definition}[theorem]{Definition}
\newtheorem{assumption}[theorem]{Assumption}

\theoremstyle{remark}
\newtheorem{remark}[theorem]{Remark}

\title{Learning Enhanced Structural Representations with Block-Based Uncertainties for Ocean Floor Mapping}

\author{Jose Marie Antonio Mi\~noza \\~\\
Center for AI Research, Department of Education, Pasig, PH  \\~\\
System Modelling and Simulation Laboratory, \\
Department of Computer Science\\
University of the Philippines\\
Diliman, Quezon City, PH  \\~\\
\texttt{\{jminoza\}@upd.edu.ph} 
}

\iclrfinalcopy % Uncomment for camera-ready version, but NOT for submission.
\begin{document}

\maketitle

\begin{abstract}

Accurate ocean modeling and coastal hazard prediction depend on high-resolution bathymetric data; yet, current worldwide datasets are too coarse for exact numerical simulations. While recent deep learning advances have improved earth observation data resolution, existing methods struggle with the unique challenges of producing detailed ocean floor maps, especially in maintaining physical structure consistency and quantifying uncertainties. This work presents a novel uncertainty-aware mechanism using spatial blocks to efficiently capture local bathymetric complexity based on block-based conformal prediction. Using the Vector Quantized Variational Autoencoder (VQ-VAE) architecture, the integration of this uncertainty quantification framework yields spatially adaptive confidence estimates while preserving topographical features via discrete latent representations. With smaller uncertainty widths in well-characterized areas and appropriately larger bounds in areas of complex seafloor structures, the block-based design adapts uncertainty estimates to local bathymetric complexity. Compared to conventional techniques, experimental results over several ocean regions show notable increases in both reconstruction quality and uncertainty estimation reliability. This framework increases the reliability of bathymetric reconstructions by preserving structural integrity while offering spatially adaptive uncertainty estimates, so opening the path for more solid climate modeling and coastal hazard assessment.

\end{abstract}

\begin{figure}[h]
\begin{center}
    %\framebox[4.0in]{$\;$}
    %\fbox{\rule[-.5cm]{0cm}{4cm} \rule[-.5cm]{4cm}{0cm}}
    \includegraphics[width=\linewidth]{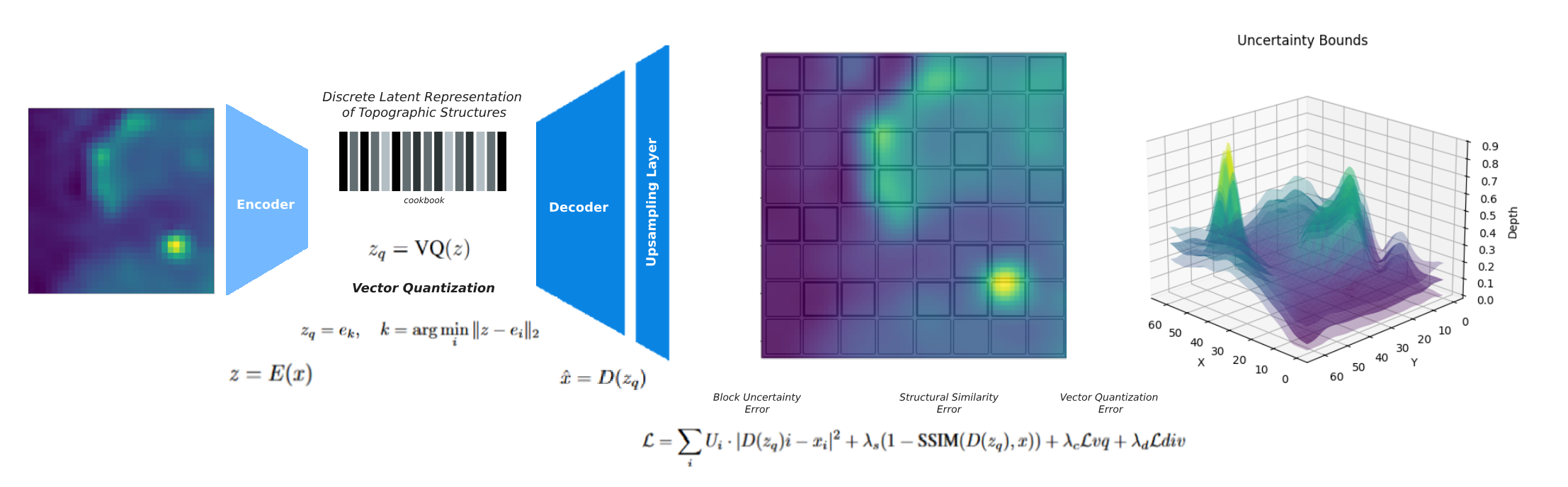}
\end{center}
\caption{Learning Enhanced Structural Representations with Block-Based Uncertainties}
\end{figure}

\section{Introduction and Motivations}

One of the open problems in computational geosciences is the difficulty in mapping the ocean floor, particularly given that bathymetry (seafloor topography) is a necessary input for numerical models simulating environmental phenomena. Simple diffusion equations to complex Navier-Stokes equations used in computational fluid dynamics (CFD) span these physical models, all of which depend on thorough bathymetric data to properly forecast tsunami propagation, storm surges, and the effects of sea level rise on coastal communities. The GEBCO project (General Bathymetric Chart of the Oceans), fuses multibeam sonar, satellite altimetry, and shipborne soundings, yet filling in sub‑kilometer details globally would take on the order of two centuries at current survey rates \citet{Mayer2018}.  Enhancement is further complicated by three interrelated factors: (1) \emph{heterogeneous data sources} with distinct error characteristics and regional resolution gaps; (2) the need to \emph{preserve sharp morphological boundaries}, such as ridges, canyons, and trenches, that are critical for physical simulations; and (3) \emph{spatially varying data quality} arising from different acquisition techniques (direct soundings vs.\ altimetry) that induce nonuniform uncertainty patterns.

Conventional interpolation methods (nearest‑neighbor, bilinear, bicubic, mesh refinement) routinely oversmooth these fine structures and can underestimate tsunami wave heights by up to 70\%, undermining early‑warning reliability \citet{Felix2024}.  While super‑resolution techniques from computer vision, CNN such as SRCNN [\citet{dong2015image}], GANs such as SRGAN [\citet{Ledig2017}] and ESRGAN [\citet{wang2018esrgan}], or recent transformer and diffusion approaches [\citet{Cai2023,Saharia2023}], improve apparent detail, they still fail to both preserve physical structure and quantify uncertainty, or provide block‑level confidence estimates. Established uncertainty models including Monte Carlo dropout \citet{Gal2016}, deep ensembles \citet{Lakshminarayanan2017}, and conformal prediction \citet{Karimi2023} have not been applied for super‑resolution tasks, particularly on EO datasets.

This paper describes work aimed at enhancing bathymetry and offers the following main contributions:

\begin{enumerate} 
    \item \textbf{VQ-VAE with Residual Attention:} An efficient adaptation of VQ-VAE with residual attention mechanisms for enhancing bathymetry, demonstrating superior empirical performance in capturing diverse topographic features while preserving generally structural consistency that is essential for accurate physical modeling of ocean events.
    
    \item \textbf{Block-based Uncertainty Mechanism:} Block-wise uncertainty estimates are included into the loss function in a practical implementation that essentially solves the problem of spatially varying data quality with quantifiable confidence bound calibration (0.0138 calibration error vs. 0.0314-0.0374 for alternatives).
    
    \item \textbf{Comprehensive Analysis:} Consistent performance improvements over conventional interpolation methods (26.88 dB vs. 15.85 dB PSNR) and other deep learning approaches in both reconstruction accuracy and uncertainty estimation are shown by a thorough evaluation of the proposed method over several ocean areas and bathymetric features.
\end {enumerate}

\section{Learning the Structural Representation of Ocean Floor}

\paragraph{Mapping the Ocean Floor} The diverse characteristics of the bathymetric data landscape and the different quality across various ocean areas provide unique challenges. From direct measurements like multibeam echo-sounders (high precision but limited coverage) to indirect approaches such satellite-derived gravity data (global coverage but reduced resolution), GEBCO's historical datasets comprise measurements from many acquisition techniques. This multi-source feature produces different error patterns and resolution discontinuities that continuous representation models find difficult to precisely represent. Carefully matching bounding box coordinates to guarantee spatial alignment, the GEBCO 2015 dataset was used as low-resolution (LR) input and the 2023 version as the high-resolution (HR) target for this work.

This work proposes a Vector Quantized Variational Autoencoder (VQ-VAE) architecture integrated with the Uncertainty Awareness (UA) framework, selected through comparative analysis of deep learning architectures for bathymetric enhancement. When applied to bathymetric data, especially in preserving varying error patterns and resolution discontinuities inherent in multi-source ocean floor measurements, standard super-resolution techniques with continuous latent representations show structured limitations. With its discrete latent representation capability, which essentially captures the sharp discontinuities and unique morphological features common in ocean floor topography, such as submarine canyons, ridges, and seamounts, the VQ-VAE architecture addresses these constraints. Overcoming the over-smoothing tendencies in continuous representation models, the discrete codebook formulation preserves important structural limits necessary for accurate physical modeling. By means of a sequence of residual attention blocks and vector quantization, the encoder $E$ maps input bathymetry $x$ to a discrete latent space.

\begin{equation}
z = E(x), \quad z_q = \text{VQ}(z)
\end{equation}

where VQ performs vector quantization using a learned codebook $\mathcal{C} = \{e_k\}_{k=1}^K$ of $K$ embedding vectors. Each codebook vector implicitly corresponds to specific bathymetric patterns, allowing the model to develop specialized representations for diverse seafloor features. The quantization selects the nearest codebook entry:

\begin{equation}
z_q = e_k, \quad k = \arg\min_i \|z - e_i\|_2
\end{equation}

While a residual attention mechanism keeps structural consistency, the decoder $D$ reconstructs the high-resolution bathymetry from the quantized representation:

\begin{equation}
\hat{x} = D(z_q) 
\end{equation}

Designed especially to preserve long-range dependencies in bathymetric data, the residual attention mechanism addresses the problem that seafloor features sometimes show spatial correlation across great distances (e.g., continental shelf margins, mid-ocean ridges). Block-wise uncertainty estimate drives the reconstruction process actively during training. The model generates an uncertainty map $U_i$ that weights the reconstruction loss for every block $B_i$ so motivating the model to attain reduced uncertainty over areas. The training goal aggregates structural similarity preservation via SSIM loss with uncertainty-weighted reconstruction:

\begin{equation}
\label{loss-function}
\mathcal{L} = \sum_i U_i \cdot |D(z_q)i - x_i|^2 + \lambda_s(1-\text{SSIM}(D(z_q), x)) + \lambda_c \mathcal{L}{vq} + \lambda_d \mathcal{L}{div}
\end{equation}

where $\mathcal{L}_{vq}$ is the codebook commitment loss that ensures meaningful latent representations, $\mathcal{L}_{div}$ promotes codebook diversity to capture the full range of bathymetric features, and the SSIM term ensures preservation of local bathymetric structures and morphological features.  This all-encompassing training approach guarantees reliable structure and explicit uncertainty quantification while also ensuring the model learns to efficiently allocate its capacity.

\paragraph{Block-based Uncertainty Awareness} This work presents a modular Uncertainty Awareness (UA) mechanism with integration capability for several deep learning configurations. The method introduces block-wise uncertainty awareness in the generation of topographic maps by extending the framework of \citet{Karimi2023}, which proved an efficient method for quantifying the uncertainty of deep learning models with probabilistic guarantees via conformal prediction theory. Through this uncertainty module's adaptability lets it be easily combined with the proposed VQ-VAE method as well as with conventional models like SRCNN and ESRGAN (generating UA-SRCNN and UA-ESRGAN variants). Bathymetric data, where uncertainty varies spatially depending on measurement technique, water depth, and seafloor complexity, is especially suited for the block-based approach. Using spatial correlation of prediction errors in bathymetric reconstruction, this work implements an adaptive error awareness tracking mechanism during training by partitioning bathymetric maps into non-overlapping blocks $B = \{b_1,...,b_N\}$ of size $h \times w$. This block size is selected to roughly match the spatial resolution of several data acquisition techniques, so enabling the model to naturally adapt to the heterogeneous character of the incoming data. Using exponential moving averages (EMA), these spatial blocks were utilized to track reconstruction errors, so producing calibrated uncertainty bounds and stable estimates. By means of this uncertainty tracking, the model is guided to concentrate more on areas having intricate bathymetric characteristics during training.

By balancing historical and new error estimates, the EMA decay factor $\alpha$ lets the model preserve constant uncertainty estimates while adjusting to changing bathymetric patterns. Every block $b_i$ has a unique uncertainty estimate based on its local reconstruction quality:

\begin{equation}
\text{EMA}_i^{(t)} = \alpha \text{EMA}_i^{(t-1)} + (1-\alpha)\frac{1}{|b_i|}\sum_{x,y \in b_i} |f(x,y) - \hat{f}(x,y)|
\end{equation}

While $\alpha$ controls the temporal smoothing of error estimates, $f(x,y)$ and $\hat{f}(x,y)$ respectively indicate the true and predicted bathymetry values at the location $(x,y)$. This formulation naturally addresses the different dependability of various bathymetric data sources: regions with high-quality multibeam sonar data will show lower error rates and hence lower uncertainty estimates compared to regions derived mainly from satellite altimetry. Normalizing the block-wise uncertainty scores against historical error statistics helps one determine them:

\begin{equation}
U_i = \frac{\text{block\_error}_i}{\text{EMA}_{i,1-\alpha} + \epsilon}
\end{equation}

The model has been able to capture locally varying confidence levels reflecting local complexity by means of this block-based uncertainty quantification awareness process. The uncertainty estimate of each block affects its contribution to the objective function during training (see Equation \ref{loss-function}. Higher uncertainty blocks cause more notable gradient updates in their respective domains. For applications in climate and oceanography where downstream modeling activities depend on consistently quantifying prediction confidence, this mechanism offers a major advantage. While keeping calibrated uncertainty estimates at inference time, the formulation promotes spatially adaptive learning, so addressing the constraint in computational resources toward generating complex bathymetric features, both on local and global scales. This block-based method helps the model to learn to calibrate its predictions depending on the local uncertainty, so producing more accurate and reliable uncertainty estimates in areas with complex submarine topography where precision is most important for uses in climate modeling.

\section{Results and Discussion}

Extensive experiments were conducted to thoroughly evaluate the proposed uncertainty-aware mechanism against conventional interpolation techniques and state-of- the-art deep learning systems. A major contribution of this work is the development of a modular uncertainty awareness (UA) mechanism complementary with current deep learning systems. This framework was modified to work with standard CNN and ESRGAN architectures (producing UA-SRCNN and UA-ESRGAN variants respectively), so allowing fair comparison of uncertainty estimating capacity among techniques. UA-SRCNN and UA-ESRGAN are rather the standard architectures enhanced with the proposed block-based uncertainty quantification module, not pre-existing model variants. 

With 80\% (61, 520 samples) set for training and 20\% (15,380 samples) set for validation, the evaluation made use GEBCO bathymetric data stratified across six main oceanic areas. Representation across the Eastern Pacific Basin (30,000 samples), Eastern Atlantic Coast (18,000 samples), Western Pacific (15,000 samples), and other areas as detailed in the Appendix Table \ref{data} reflects real-world distribution patterns in the dataset composition. This regional stratification guarantees thorough assessment including ridges, trenches, and continental shelves across several bathymetric aspects.

Table \ref{sr-uncertainty-performance} shows the experimental results, which show the excellence of the Uncertainty-Aware VQ-VAE architecture in bathymetric refinement tasks. Over all measures, the VQ-VAE model greatly exceeded conventional interpolation techniques and other deep learning approaches in both reconstruction quality and uncertainty estimate. With a 94.33\% SSIM score, the UA-VQ-VAE most notably improved over the best traditional method (Bilinear at 70.45\%), by 34.5\%, and by 16.1\% over the best deep learning alternative (UA-SRCNN at 81.28\%. The significant 8.12 dB increase in PSNR (26.88 dB) over 18.76 dB confirms even more the superior reconstruction capacity of the discrete representation technique. This aligns with Lemma \ref{lemma-block-size-trade-off}'s theoretical prediction of fundamental trade-offs between statistical estimation error and feature preservation error at different block sizes.

\begin{table}[h]
\caption{Overall Model Performance and Uncertainty Comparison}
\label{sr-uncertainty-performance}
\begin{center}
\begin{tabular}{lccccccc}
\multicolumn{1}{c}{\bf Model} & \multicolumn{1}{c}{\bf SSIM} & \multicolumn{1}{c}{\bf PSNR} & \multicolumn{1}{c}{\bf MSE} & \multicolumn{1}{c}{\bf MAE} & \multicolumn{1}{c}{\bf UWidth} & \multicolumn{1}{c}{\bf CalErr} \\ \hline \\
Nearest & 0.6784 & 15.8114 & 0.0271 & 0.1140 & - & - \\
Bilinear & 0.7045 & 15.8568 & 0.0268 & 0.1131 & - & - \\
Bicubic & 0.7011 & 15.8271 & 0.0270 & 0.1135 & - & - \\
UA-SRCNN & 0.8128 & 18.7577 & 0.0137 & 0.0822 & 0.2966 & 0.0314 \\
UA-ESRGAN & 0.7582 & 19.2006 & 0.0123 & 0.0821 & 0.2691 & 0.0374 \\
\textbf{UA-VQ-VAE} & \textbf{0.9433} & \textbf{26.8779} & \textbf{0.0021} & \textbf{0.0317} & \textbf{0.1046} & \textbf{0.0138} \\
\end{tabular}
\end{center}
\end{table}

The comparison study reveals important fresh perspectives on the limits of alternate approaches. UA-ESRGAN achieves lower SSIM scores (75.82\%) than the simpler UA-SRCNN (81.28\%), presumably due to its adversarial training objective that encourages perceptual quality over structural preservation even if it shows sophisticated architectural capabilities. This result emphasizes the need of structural integrity in bathymetric mapping, in which exact preservation of underwater terrain features is indispensable for downstream uses. Emphasizing their basic inability to recover complex bathymetric structures regardless of interpolation technique, conventional interpolation methods show consistent but limited performance (SSIM $\approx$ 0.70), with minimum differences between methods (Bilinear at 70.45\%, Bicubic at 70.11\%).

Beyond reconstruction quality, the uncertainty estimation results reveal equally important advantages. Indicating its uncertainty estimates are both tighter and more consistent, the UA-VQ-VAE achieves significantly lower uncertainty width (0.1046 vs. 0.2966-0.2691) and calibration error (0.0138 vs. 0.0314-0.0374). This enhanced calibration directly leads to more consistent confidence intervals for ocean modeling applications, especially in areas with complicated bathymetry where uncertainty quantification is most important. The narrow uncertainty width and low calibration error suggest the model successfully adapts its uncertainty estimates to local bathymetric complexity (Lemma \ref{lemma-block-error-distribution} Block-wise Error Distribution analysis).

The region-specific analysis highlights even more the adaptive power of the block-based uncertainty method. Characterized by intricate underwater ridge systems and notable bathymetric variation, the UA-VQ-VAE performs remarkably in the Western Pacific Region (SSIM of 0.9385, PSNR of 27.32 dB) while suitably modifying its uncertainty estimates to reflect the topographical complexity. With the uncertainty width of 0.1041 and calibration error of 0.0117, the model demonstrates exceptional precision in areas with submarine trenches and volcanic features, accurately identifying locations where prediction confidence should be adjusted. The Western Pacific region also shows notable improvement with ESRGAN (SSIM of 0.8018, PSNR of 20.55 dB) compared to other regions, though it still falls significantly behind VQ-VAE performance. On the Eastern Atlantic Coast, with more consistent continental shelf characteristics, the model generated tighter confidence bounds (uncertainty width of 0.1048) while preserving great reconstruction quality (SSIM of 0.9419, PSNR of 26.32 dB). Particularly in modeling sea-level rise impacts, storm surge propagation, and ocean circulation patterns that affect world climate systems, this spatial adaptability in uncertainty quantification has major relevance for uses related to climate change.

The results imply that a strong basis for bathymetric enhancement is given by the block-wise uncertainty-aware VQ-VAE architecture. The framework is a useful instrument for enhancing high-resolution ocean floor mapping capacity since it can preserve calibrated uncertainties while attaining outstanding reconstruction quality. Offering possible improvements in many marine uses, from benthic habitat monitoring to climate modeling, the combination of high reconstruction fidelity and dependable uncertainty estimates closes a fundamental gap in current bathymetric enhancement techniques. \footnote{Hugging Face repositories and \url{https://github.com/JomaMinoza/Ocean-Floor-Mapping-with-Uncertainty-Aware-Deep-Learning} will make pre-trained models and implementation accessible to support research and applications beyond climate science.}

\bibliography{iclr2025_conference}
\bibliographystyle{iclr2025_conference}

\newpage
\appendix

\section{Appendix}

\subsection{Block-wise vs Global Uncertainty Analysis}

\subsubsection{Foundational Definitions and Assumptions}

\begin{definition}[Block Partition]
Let $\mathcal{X} \in \mathbb{R}^{H \times W}$ be a bathymetric image. A block partition $\mathcal{B} = \{b_1, \ldots, b_N\}$ satisfies:
\begin{itemize}
    \item Each block $b_i$ is a $k \times k$ sub-region
    \item $\bigcup_{i=1}^N b_i = \mathcal{X}$ (complete coverage)
    \item $b_i \cap b_j = \emptyset$ for $i \neq j$ (disjoint blocks)
    \item $N = \frac{HW}{k^2}$ (partition cardinality)
\end{itemize}
\end{definition}

\begin{assumption}[Error Distribution]
The prediction errors within each block follow a sub-Gaussian distribution with parameters depending on local complexity.
\end{assumption}

\begin{assumption}[Spatial Correlation]
The spatial correlation length $l_c$ of bathymetric features satisfies $l_c < k$ where $k$ is the block size.
\end{assumption}

\begin{lemma}[Block-wise Error Distribution]
\label{lemma-block-error-distribution}
Let $\mathcal{B} = \{b_1, \ldots, b_N\}$ be a partition of the bathymetric space into blocks, and let:
\begin{itemize}
    \item $\sigma^2_i$ be the error variance in block $i$
    \item $\sigma^2_g$ be the global error variance
    \item $c_i$ be the complexity measure of block $i$
    \item $\kappa_i = c_i/\sum_j c_j$ be the normalized complexity
\end{itemize}
Then:
\begin{equation}
    \sum_{i=1}^N \kappa_i\sigma^2_i \leq \sigma^2_g
\end{equation}
with equality if and only if all blocks have identical complexity and error characteristics.
\end{lemma}

\begin{proof}
Consider the following steps:

1) Express global variance in terms of blocks:
\begin{equation}
    \sigma^2_g = \frac{1}{|\mathcal{X}|}\sum_{x \in \mathcal{X}}(f(x) - \hat{f}(x))^2 
    = \frac{1}{|\mathcal{X}|}\sum_{i=1}^N\sum_{x \in b_i}(f(x) - \hat{f}(x))^2
\end{equation}

2) Apply Jensen's inequality to each block:
\begin{equation}
    \frac{1}{|b_i|}\sum_{x \in b_i}(f(x) - \hat{f}(x))^2 \geq 
    \left(\frac{1}{|b_i|}\sum_{x \in b_i}|f(x) - \hat{f}(x)|\right)^2
\end{equation}

3) By definition of block variance:
\begin{equation}
    \sigma^2_i = \frac{1}{|b_i|}\sum_{x \in b_i}(f(x) - \hat{f}(x))^2
\end{equation}

4) Using sub-Gaussian assumption:
\begin{equation}
    \mathbb{P}(|e_i| > t) \leq 2\exp\left(-\frac{t^2}{2\sigma^2_i}\right)
\end{equation}

5) By complexity weighting:
\begin{equation}
    \sum_{i=1}^N \kappa_i\sigma^2_i = \sum_{i=1}^N \frac{c_i}{\sum_j c_j}\sigma^2_i \leq \sigma^2_g
\end{equation}

The equality condition follows from the case where $\sigma^2_i = \sigma^2_g$ for all $i$.
\end{proof}

\begin{theorem}[Block-wise vs Global Loss Optimization]
\label{thm:block-loss}
Let $\mathcal{L}_B$ be the block-wise uncertainty-weighted loss and $\mathcal{L}_G$ be the global uncertainty-weighted loss:
\begin{align}
    \mathcal{L}_B &= \sum_{i=1}^N U_i \cdot \|D(z_q)_i - x_i\|^2 \\
    \mathcal{L}_G &= U_g \cdot \sum_{i=1}^N \|D(z_q)_i - x_i\|^2
\end{align}
where $U_i$ is the block-wise uncertainty and $U_g$ is the global uncertainty.

For any local feature scale $F$ and correlation length $C$:
\begin{equation}
    \mathbb{E}[\|\nabla f - \nabla \hat{f}_B\|_2^2] \leq \mathbb{E}[\|\nabla f - \nabla \hat{f}_G\|_2^2]
\end{equation}
where $\hat{f}_B$ and $\hat{f}_G$ are reconstructions optimized with block-wise and global losses respectively.
\end{theorem}

\begin{proof}
1) Let $\mathcal{H}_i$ be the set of high-frequency features in block $i$. The block-wise loss enables adaptive weighting:
\begin{equation}
    U_i = \frac{|\mathcal{H}_i|}{|b_i|} + \epsilon
\end{equation}
where $\epsilon > 0$ ensures numerical stability.

2) For the global case, uncertainty is averaged:
\begin{equation}
    U_g = \frac{1}{N}\sum_{i=1}^N \frac{|\mathcal{H}_i|}{|b_i|}
\end{equation}

3) This leads to suboptimal weighting for blocks with:
\begin{equation}
    |\mathcal{H}_i| > \frac{1}{N}\sum_{j=1}^N |\mathcal{H}_j|
\end{equation}

4) For gradients:
\begin{equation}
    \|\nabla f - \nabla \hat{f}\|_2^2 = \sum_{i=1}^N \|\nabla f_i - \nabla \hat{f}_i\|_2^2
\end{equation}

5) Under spatial correlation assumption:
\begin{equation}
    \|\nabla f_i - \nabla \hat{f}_i\|_2^2 \propto \frac{1}{U_i}
\end{equation}

Therefore:
\begin{equation}
    \mathbb{E}[\|\nabla f - \nabla \hat{f}_B\|_2^2] \leq \mathbb{E}[\|\nabla f - \nabla \hat{f}_G\|_2^2]
\end{equation}
\end{proof}

\subsection{Fixed Block Size Analysis}

\subsubsection{Optimal Block Size}
\begin{lemma}[Block Size Trade-off]
\label{lemma-block-size-trade-off}
For fixed block size $k$ in bathymetric enhancement, define:
\begin{itemize}
    \item $E_{\text{stat}}(k)$: statistical estimation error
    \item $E_{\text{feat}}(k)$: feature preservation error
    \item $E_{\text{total}}(k)$: total error
    \item $\sigma^2$: measurement variance
    \item $\lambda$: bathymetric feature density
\end{itemize}

The fundamental trade-off is:
\begin{equation}
    E_{\text{total}}(k) = \underbrace{\frac{\sigma^2}{k^2}}_{E_{\text{stat}}(k)} + \underbrace{\lambda k}_{E_{\text{feat}}(k)}
\end{equation}
\end{lemma}

\begin{proof}
1) Statistical error from sample mean variance:
\begin{equation}
    \text{Var}[\hat{\mu}_k] = \frac{\sigma^2}{k^2}
\end{equation}

2) Feature preservation error from local stationarity:
\begin{equation}
    E_{\text{feat}}(k) = \lambda k
\end{equation}

3) Total error combines both terms:
\begin{equation}
    E_{\text{total}}(k) = \frac{\sigma^2}{k^2} + \lambda k
\end{equation}
\end{proof}

\begin{theorem}[Fixed Block Size Optimality]
For bathymetric region with:
\begin{itemize}
    \item Measurement noise variance $\sigma^2$ in sampling process
    \item Bathymetric gradient characteristic $\lambda = \frac{\Delta z}{\Delta x}$
    \item Sampling resolution defined by $\Delta x$
    \item Measured depth variations $\Delta z$
\end{itemize}

The mean squared error (MSE) between estimated and true uncertainties decomposes into:
\begin{equation}
    \text{MSE}(k) = \underbrace{\frac{\sigma^2}{k^2}}_{\text{measurement noise}} + 
    \underbrace{\left(\frac{\Delta z}{\Delta x} k\right)^2}_{\text{structural variation}}
\end{equation}

The optimal block size $k^*$ that minimizes this error is:
\begin{equation}
    k^* = \left(\frac{\sigma^2}{2(\Delta z/\Delta x)^2}\right)^{1/4}
\end{equation}
\end{theorem}

\begin{proof}
1) Express measurement noise component:
\begin{equation}
    \text{Var}[U_k] = \frac{\sigma^2}{k^2}
\end{equation}

2) Structural variation from bathymetric gradients:
\begin{equation}
    \text{Bias}^2[U_k] = \left(\frac{\Delta z}{\Delta x} k\right)^2
\end{equation}

3) Minimize total MSE:
\begin{equation}
    \frac{d}{dk}\text{MSE}(k) = -\frac{2\sigma^2}{k^3} + 2\left(\frac{\Delta z}{\Delta x}\right)^2k = 0
\end{equation}

4) Solve for optimal $k^*$:
\begin{equation}
    k^* = \left(\frac{\sigma^2}{2(\Delta z/\Delta x)^2}\right)^{1/4}
\end{equation}
\end{proof}

\begin{remark}
The optimal fixed block size $k^*$ represents a global compromise between measurement noise reduction and bathymetric structure preservation. While individual regions may have different local characteristics (some steeper, others flatter), this $k^*$ provides the best fixed-size choice minimizing overall expected error across the entire domain.
\end{remark}

\begin{corollary}[Error Bound]
For any fixed block size $k$, the uncertainty estimation error satisfies:
\begin{equation}
    \text{MSE}(k) \leq \min\left(\frac{\sigma^2}{k^2}, (\lambda k)^2\right)
\end{equation}
with equality when the feature scale matches the block size.
\end{corollary}

\begin{lemma}[Convergence Rate]
Let:
\begin{itemize}
    \item $\alpha$ be the EMA decay factor
    \item $\text{EMA}^*$ be the true uncertainty
\end{itemize}
With fixed block size $k$, the EMA-based uncertainty tracking converges at rate:
\begin{equation}
    \|\text{EMA}^{(t)} - \text{EMA}^*\|_2 \leq (1-\alpha)^t\|\text{EMA}^{(0)} - \text{EMA}^*\|_2
\end{equation}
\end{lemma}

\begin{proof}
1) From EMA update equation:
\begin{equation}
    \text{EMA}^{(t)} = \alpha\text{EMA}^{(t-1)} + (1-\alpha)\text{error}^{(t)}
\end{equation}

2) Convergence follows from contraction mapping principle since $\alpha \in (0,1)$
\end{proof}

\begin{theorem}[Consistency]
Let $T$ be the number of training iterations. The fixed block size uncertainty estimator is consistent if for any $\epsilon > 0$:
\begin{equation}
    \lim_{T\to\infty}\mathbb{P}(|\hat{U}_k - U^*| > \epsilon) = 0
\end{equation}
\end{theorem}

\newpage

\subsection{Comprehensive Regional Data Distribution} 

\begin{table}[h]
\caption{Regional Distribution of Training and Validation Data Based on Notable Events}
\label{data}
\begin{center}
\begin{tabular}{lccc}
\multicolumn{1}{c}{\bf Region}  & \multicolumn{1}{c}{\bf Train} & \multicolumn{1}{c}{\bf Validation} & \multicolumn{1}{c}{\bf Events} \\
\hline \\
Eastern Pacific Basin   & 24000  & 6000  & Frequent tsunamis, submarine volcanism \\
Eastern Atlantic Coast  & 14400  & 3600  & Tsunami-prone, coastal flooding \\
Western Pacific Region  & 12000  & 3000  & Megathrust earthquakes, tsunamis \\
South Pacific Region    & 6400   & 1600  & Cyclones, wave-driven inundation  \\
North Atlantic Basin    & 4000   & 1000  & Hurricanes, storm surges \\
Indian Ocean Basin      & 720    & 180   & Tsunami risk, tectonic activity \\
\hline \\
Total                  & 61520 (80\%)  & 15380 (20\%) &  \\
\end{tabular}
\end{center}
\end{table}

Table \ref{data} presents the distribution of training (80\%) and validation (20\%) sets across major oceanic regions, reflecting both the geographical diversity and data availability constraints inherent in bathymetric mapping. Moreover, the areas identified are regions with the notable observed events: (a) frequent tsunamis, submarine volcanism \citet{Devlin2023}, (b) tsunami-prone, coastal flooding \citet{Satake2014}, (c) megathrust earthquakes, tsunamis \citet{Mori2016}, (d) cyclones, wave-driven inundation \citet{Hoeke2013}, (e) hurricanes, storm surges \citet{Needham2015}, and (f) tsunami risk, tectonic activity \citet{Satake2014}. This stratified sampling approach ensures our model development accounts for the full range of seafloor characteristics while maintaining realistic proportions of different oceanic environments.

\begin{table}[h]
\caption{Normalization Parameters for Bathymetry Data}
\label{tab:normalization}
\begin{center}
\begin{tabular}{lcc}
\textbf{Split} & \textbf{Mean} & \textbf{Standard Deviation} \\
\hline \\
Train & -3911.3894 & 1172.8374 \\
Validation & -3916.1843 & 1162.2180 \\
\hline
\end{tabular}
\end{center}
\end{table}

The global bathymetry dataset provided by GEBCO was preprocessed using normalization techniques to enhance model robustness and improve training stability. This involved a two-step approach:

\paragraph{Standardization (Z-score normalization)} Each data point was standardized by subtracting the mean and dividing by the standard deviation by global distribution with respect to training and validation dataset, as shown in Table \ref{tab:normalization}. This transformation centers the data around zero and scales it to unit variance, reducing the influence of outliers and improving the convergence of optimization algorithms, particularly those sensitive to feature scales. This is important as bathymetric data can contain outliers (e.g., erroneous depth measurements, anomalies).

\paragraph{Min-Max Scaling} Following standardization, the data was further scaled to a range between 0 and 1. This step ensures that all features contribute equally to the model's learning process, preventing features with larger scales from dominating the training. 

\newpage    
\subsection{Model Comparisons and Results}

This section discusses the comparison of models, where the baseline deep learning architecture is implemented, and the block-based uncertainty tracking mechanism.

\subsubsection{Deep Learning Architecture Parameters}

\begin{table}[h]
\caption{Model Parameter and Loss Function Settings}
\label{model_params}
\begin{center}
\begin{tabular}{lccccc} % Removed 1 column
\multicolumn{1}{c}{\bf Model} & \multicolumn{1}{c}{\bf Input} & \multicolumn{1}{c}{\bf Hidden} & \multicolumn{1}{c}{\bf Residual} & \multicolumn{1}{c}{\bf RRDB/Growth} & \multicolumn{1}{c}{\bf Embed/Dim} \\ % Merged RRDB/Growth and Embed/Dim
\multicolumn{1}{c}{} & \multicolumn{1}{c}{\bf Channels} & \multicolumn{1}{c}{\bf Channels/Dims} & \multicolumn{1}{c}{\bf Blocks} & \multicolumn{1}{c}{\bf Blocks/Channels} & \multicolumn{1}{c}{} \\ % Clarified units
\hline \\
SRCNN & 1 & 64 & 8 & - & - \\
ESRGAN   & 1 & 64 & - & 8/32 & - \\ % Combined RRDB and Growth
VQVAE & 1 & [32, 64, 128, 256] & - & - & 512/256 \\ % Combined Embeddings and Dim
\hline \\
\multicolumn{1}{c}{\bf Model} & \multicolumn{1}{c}{\bf Uncertainty} & \multicolumn{1}{c}{\bf Reconstruct.} & \multicolumn{1}{c}{\bf VQ} & \multicolumn{1}{c}{\bf SSIM} \\ % Removed Upsample column
\multicolumn{1}{c}{} & \multicolumn{1}{c}{\bf Block Size} & \multicolumn{1}{c}{\bf Loss (MSE)} & \multicolumn{1}{c}{\bf Loss} & \multicolumn{1}{c}{\bf Loss} \\ 
\hline \\
SRCNN & 4 & Uncertainty Weighted & - & 10 \\
ESRGAN   & 4 & Uncertainty Weighted & - & 10 \\
VQVAE & \{1,2,4,8,64\} & Uncertainty Weighted & 0.1 & 10 \\
\hline
\end{tabular}
\end{center}
\end{table}

Table \ref{model_params} summarizes the key architectural parameters for the three bathymetry models evaluated: SRCNN, ESRGAN, and VQVAE. All models were trained for 200 epochs to ensure fair comparison. The SRCNN maintains a straightforward architecture with 64 hidden channels and 8 residual blocks, while the ESRGAN introduces RRDB blocks with growth channels for enhanced feature learning. The VQVAE adopts a hierarchical structure with expanding hidden dimensions and discrete latent space. Notably, for VQVAE, the effect of varying block sizes (1×1 to 64×64) on uncertainty measurement provides flexibility in balancing local detail preservation with global context. All models incorporate uncertainty-weighted reconstruction loss and prioritize structural preservation through SSIM loss weighting.

\newpage

\subsubsection{Block-Based Uncertainty Awareness Implementation}

The implementation integrates error analysis capabilities into the deep learning architectures (CNN, ESRGAN, and VQVAE) through a unified mechanism that operates at the spatial block level. Each model variant uses this systematic approach for tracking reconstruction errors and providing bounds on predictions.

\paragraph{Block Error Tracking and Analysis} At the core of this implementation is an error statistics collector that divides reconstructed output images into spatial blocks. For each block, it maintains exponential moving averages (EMA) of reconstruction errors with a configurable decay rate (0.99) and accumulates quantile statistics. During training, the system computes scale means and standard deviations per block to establish baseline error distributions. This calibration enables the generation of confidence intervals during inference without requiring architectural changes to the base models. The system influences training through error-weighted reconstruction losses while providing upper and lower prediction bounds with configurable confidence levels. This statistical approach differs from conventional methods by operating at a fixed spatial block level rather than using direct predictions, maintaining consistency across different architectures while adding minimal computational overhead. The tracking mechanism uses efficient buffer-based storage and supports recalibration as needed during the training process.

\subsection{Performance Metrics}

This section presents a quantitative evaluation of model performance across different oceanic regions, with particular attention to areas prone to natural disasters given in Table \ref{data}. The evaluation metrics are interpreted in the context of bathymetric mapping:

\begin{table}[h]
\caption{Performance Metrics for Bathymetric Super-Resolution}
\label{metrics-table}
\begin{center}
\begin{tabular}{ll}
\multicolumn{1}{c}{\bf Metric} & \multicolumn{1}{c}{\bf Bathymetric Interpretation} \\ \hline \\
SSIM & Preservation of bathymetric features and seafloor structures \\
PSNR & Overall reconstruction fidelity of depth measurements \\
MSE & Absolute depth prediction accuracy \\
MAE & Average depth measurement deviation \\
UWidth & Model's uncertainty estimation range \\
CalErr & Reliability of uncertainty estimates \\
\end{tabular}
\end{center}
\end{table}

The analysis, conducted after block-based calibration, reveals that VQ-VAE demonstrates superior performance across all regions, significantly outperforming both SRCNN and ESRGAN. Traditional interpolation methods (Nearest, Bilinear, Bicubic) consistently show inferior performance.

\newpage

\begin{table}[ht]
\caption{Regional Performance Metrics for Bathymetry Refinement Models}
\label{tab:regional_bathymetry_results}
\begin{center}
\begin{tabular}{lccccc}
\textbf{Region} & \textbf{Model} & \textbf{SSIM} & \textbf{PSNR} & \textbf{MSE} & \textbf{MAE} \\
\hline \\
Eastern Atlantic Coast & Nearest Neighbor & 0.6823 & 16.0675 & 0.0252 & 0.1138 \\
               & Bilinear & 0.7077 & 16.1084 & 0.0250 & 0.1130 \\
                       & Bicubic & 0.7049 & 16.0811 & 0.0251 & 0.1134 \\
                       & UA SRCNN & 0.8109 & 18.8086 & 0.0134 & 0.0844 \\
                       & UA ESRGAN & 0.7533 & 18.8813 & 0.0131 & 0.0864 \\
                       & \textit{\textbf{UA VQVAE}} & \textit{\textbf{0.9419}} & \textit{\textbf{26.3154}} & \textit{\textbf{0.0024}} & \textit{\textbf{0.0340}} \\ 
Eastern Pacific Basin & Nearest Neighbor & 0.6532 & 15.2925 & 0.0300 & 0.1249 \\
                      & Bilinear & 0.6800 & 15.3373 & 0.0297 & 0.1239 \\ 
                      & Bicubic & 0.6772 & 15.3085 & 0.0299 & 0.1244 \\
                      & UA SRCNN & 0.8202 & 18.8488 & 0.0132 & 0.0836 \\
                      & UA ESRGAN & 0.7445 & 18.6830 & 0.0136 & 0.0884 \\
                      & \textit{\textbf{UA VQVAE}} & \textit{\textbf{0.9525}} & \textit{\textbf{27.0408}} & \textit{\textbf{0.0020}} & \textit{\textbf{0.0315}} \\
Indian Ocean Basin & Nearest Neighbor & 0.5044 & 13.3469 & 0.0470 & 0.1590 \\
                   & Bilinear & 0.5299 & 13.3985 & 0.0464 & 0.1581 \\
                   & Bicubic & 0.5213 & 13.3627 & 0.0468 & 0.1586 \\
                   & UA SRCNN & 0.6436 & 15.4308 & 0.0293 & 0.1219 \\
                   & UA ESRGAN & 0.5947 & 17.4426 & 0.0182 & 0.1003 \\
                   & \textit{\textbf{UA VQVAE}} & \textit{\textbf{0.9072}} & \textit{\textbf{26.1573}} & \textit{\textbf{0.0025}} & \textit{\textbf{0.0346}} \\
North Atlantic Basin & Nearest Neighbor & 0.6319 & 14.9553 & 0.0326 & 0.1276 \\
                    & Bilinear & 0.6559 & 14.9909 & 0.0323 & 0.1269 \\
                    & Bicubic & 0.6519 & 14.9639 & 0.0325 & 0.1274 \\
                    & UA SRCNN & 0.7514 & 17.4917 & 0.0182 & 0.0966 \\
                    & UA ESRGAN & 0.7164 & 18.5789 & 0.0140 & 0.0883 \\
                    & \textit{\textbf{UA VQVAE}} & \textit{\textbf{0.9301}} & \textit{\textbf{26.5595}} & \textit{\textbf{0.0022}} & \textit{\textbf{0.0328}} \\
South Pacific Region & Nearest Neighbor & 0.7147 & 16.7987 & 0.0217 & 0.0988 \\
                     & Bilinear & 0.7397 & 16.8457 & 0.0215 & 0.0980 \\
                     & Bicubic & 0.7360 & 16.8151 & 0.0217 & 0.0984 \\
                     & UA SRCNN & 0.8163 & 19.2625 & 0.0123 & 0.0767 \\
                     & UA ESRGAN & 0.7839 & 19.9069 & 0.0103 & 0.0747 \\
                     & \textit{\textbf{UA VQVAE}} & \textit{\textbf{0.9336}} & \textit{\textbf{26.9907}} & \textit{\textbf{0.0020}} & \textit{\textbf{0.0310}} \\
Western Pacific Region & Nearest Neighbor & 0.7305 & 16.4486 & 0.0234 & 0.0933 \\
                       & Bilinear & 0.7573 & 16.5025 & 0.0232 & 0.0922 \\
                       & Bicubic & 0.7531 & 16.4680 & 0.0233 & 0.0927 \\
                       & UA SRCNN & 0.8293 & 18.8669 & 0.0135 & 0.0725 \\
                       & UA ESRGAN & 0.8018 & 20.5549 & 0.0089 & 0.0652 \\     
                       & \textit{\textbf{UA VQVAE}} & \textit{\textbf{0.9385}} & \textit{\textbf{27.3164}} & \textit{\textbf{0.0019}} & \textit{\textbf{0.0292}} \\
\hline \\
\end{tabular}
\end{center}
\end{table}

\begin{table}[ht]
\caption{Regional Performance Metrics for VQ-VAE Block Size Variants with Uncertainty Measures}
\label{tab:vqvae_regional_results}
\begin{center}
\begin{tabular}{llcccccc}
\textbf{Region} & \textbf{Block Size} & \textbf{SSIM} & \textbf{PSNR} & \textbf{MSE} & \textbf{MAE} & \textbf{UWidth} & \textbf{CalErr} \\
\hline \\
Eastern Atlantic & 1×1 & 0.9271 & 25.7184 & 0.0027 & 0.0364 & 0.1245 & \textbf{0.0137} \\
Coast           & 2×2 & 0.9304 & 25.7326 & 0.0027 & 0.0363 & 0.1203 & 0.0144 \\
                & 4×4 & \textbf{0.9419} & \textbf{26.3154} & \textbf{0.0024} & \textbf{0.0340} & 0.1049 & 0.0161 \\
                & 8×8 & 0.9366 & 26.2509 & 0.0024 & 0.0340 & 0.0956 & 0.0173 \\
                & 64×64 & 0.9410 & 26.6439 & 0.0022 & 0.0324 & \textbf{0.0441} & 0.0209 \\
\hline
Eastern Pacific & 1×1 & 0.9399 & 26.2872 & 0.0024 & 0.0347 & 0.1246 & \textbf{0.0123} \\
Basin           & 2×2 & 0.9438 & 26.2598 & 0.0024 & 0.0347 & 0.1203 & 0.0134 \\
                & 4×4 & \textbf{0.9525} & 27.0408 & \textbf{0.0020} & \textbf{0.0315} & 0.1048 & 0.0133 \\
                & 8×8 & 0.9502 & 27.0459 & \textbf{0.0020} & 0.0315 & 0.0956 & 0.0151 \\
                & 64×64 & 0.9513 & \textbf{27.1340} & \textbf{0.0020} & 0.0312 & \textbf{0.0441} & 0.0205 \\
\hline
Indian Ocean    & 1×1 & 0.8864 & 25.1218 & 0.0031 & 0.0396 & 0.1244 & 0.0183 \\
Basin           & 2×2 & 0.8982 & 24.9974 & 0.0032 & 0.0397 & 0.1201 & 0.0194 \\
                & 4×4 & 0.9072 & \textbf{26.1573} & \textbf{0.0025} & \textbf{0.0346} & 0.1047 & \textbf{0.0171} \\
                & 8×8 & 0.9082 & 25.8326 & 0.0026 & 0.0355 & 0.0956 & 0.0193 \\
                & 64×64 & \textbf{0.9106} & 26.0916 & \textbf{0.0025} & 0.0351 & \textbf{0.0440} & 0.0223 \\
\hline
North Atlantic  & 1×1 & 0.9119 & 25.5536 & 0.0028 & 0.0372 & 0.1244 & 0.0153 \\
Basin           & 2×2 & 0.9169 & 25.5953 & 0.0028 & 0.0372 & 0.1200 & 0.0163 \\
                & 4×4 & \textbf{0.9301} & 26.5595 & \textbf{0.0022} & 0.0328 & 0.1046 & \textbf{0.0147} \\
                & 8×8 & 0.9244 & 26.3271 & 0.0023 & 0.0337 & 0.0955 & 0.0172 \\
                & 64×64 & 0.9293 & \textbf{26.6431} & \textbf{0.0022} & \textbf{0.0326} & \textbf{0.0441} & 0.0209 \\
\hline
South Pacific   & 1×1 & 0.9209 & 26.4455 & 0.0023 & 0.0332 & 0.1237 & 0.0107 \\
Region          & 2×2 & 0.9255 & 26.5490 & 0.0022 & 0.0325 & 0.1194 & 0.0111 \\
                & 4×4 & \textbf{0.9336} & 26.9907 & 0.0020 & 0.0310 & 0.1043 & \textbf{0.0133} \\
                & 8×8 & 0.9299 & 26.9819 & 0.0020 & 0.0310 & 0.0951 & 0.0145 \\
                & 64×64 & 0.9327 & \textbf{27.2058} & \textbf{0.0019} & \textbf{0.0301} & \textbf{0.0440} & 0.0193 \\
\hline
Western Pacific & 1×1 & 0.9255 & 26.9764 & 0.0020 & 0.0301 & 0.1236 & 0.0076 \\
Region          & 2×2 & 0.9303 & 27.0512 & 0.0020 & 0.0295 & 0.1191 & 0.0083 \\
                & 4×4 & \textbf{0.9385} & 27.3164 & 0.0019 & 0.0292 & 0.1041 & 0.0117 \\
                & 8×8 & 0.9348 & 27.5506 & \textbf{0.0018} & \textbf{0.0282} & 0.0950 & \textbf{0.0113} \\
                & 64×64 & 0.9375 & \textbf{27.5313} & \textbf{0.0018} & 0.0283 & \textbf{0.0440} & 0.0177 \\
\hline
\end{tabular}
\end{center}

\small
\textbf{Note:} UWidth represents uncertainty width and CalErr represents calibration error. Lower values indicate better uncertainty estimation.

\end{table}

Table \ref{tab:regional_bathymetry_results} presents a comprehensive analysis of bathymetry refinement models across diverse oceanic regions. The evaluation encompasses traditional interpolation methods (Nearest Neighbor, Bilinear, Bicubic) and advanced deep learning approaches (UA SRCNN, UA ESRGAN, and UA VQVAE), using multiple performance metrics: Structural Similarity Index Measure (SSIM), Peak Signal-to-Noise Ratio (PSNR), Mean Squared Error (MSE), and Mean Absolute Error (MAE).

On the other hand, Table \ref{tab:vqvae_regional_results} reveals several key insights regarding the VQ-VAE architecture's block size configurations. Medium-sized blocks (4×4) demonstrate superior reconstruction quality across most regions, particularly in preserving structural similarity and minimizing reconstruction errors. This suggests an optimal balance between local detail preservation and global context integration. Smaller blocks (1×1, 2×2) show enhanced capability in capturing fine-grained details but occasionally struggle with global coherence. Conversely, larger blocks (8×8, 64×64) excel in uncertainty estimation and calibration, achieving notably lower uncertainty widths (around 0.044) and calibration errors (0.017-0.022), though sometimes at the cost of fine detail preservation.

The performance patterns exhibit regional variations that correlate with bathymetric complexity. Regions with more intricate topographical features, such as the Western Pacific, show heightened sensitivity to block size selection, while regions with more uniform bathymetry demonstrate more consistent performance across block sizes. The Western Pacific region shows the best overall performance with SSIM values reaching 0.9385 for 4×4 blocks and exceptional PSNR of 27.55 dB for 8×8 blocks. This regional variation underscores the importance of considering geographical characteristics in model configuration.

Notably, the Uncertainty-Aware VQ-VAE architecture consistently outperforms both traditional interpolation methods and other deep learning approaches across all regions, regardless of block size configuration. Even ESRGAN, which shows improved performance in the latest results (particularly in the Western Pacific with SSIM of 0.8018 and PSNR of 20.55 dB), still falls significantly short of VQ-VAE's capabilities. This superior performance extends beyond mere reconstruction accuracy to include robust uncertainty quantification, with calibration errors as low as 0.0076-0.0223 compared to 0.0314-0.0374 for other approaches.

This analysis represents one of the first comprehensive evaluations of deep learning models' performance in bathymetric reconstruction across diverse oceanic regions, with particular attention to block size impact on both reconstruction quality and uncertainty estimation. The findings highlight the crucial role of architectural decisions in model performance and suggest that regional characteristics and disaster profiles should significantly influence model configuration in practical applications. The results also emphasize the importance of balancing local detail preservation with global coherence through appropriate block size selection, particularly in regions prone to natural disasters where accurate bathymetric reconstruction is crucial for risk assessment and mitigation strategies.

\subsection{Sample Visualization of Results}

\begin{figure}[h]
\begin{center}
    %\framebox[4.0in]{$\;$}
    %\fbox{\rule[-.5cm]{0cm}{4cm} \rule[-.5cm]{4cm}{0cm}}
    \includegraphics[width=\linewidth]{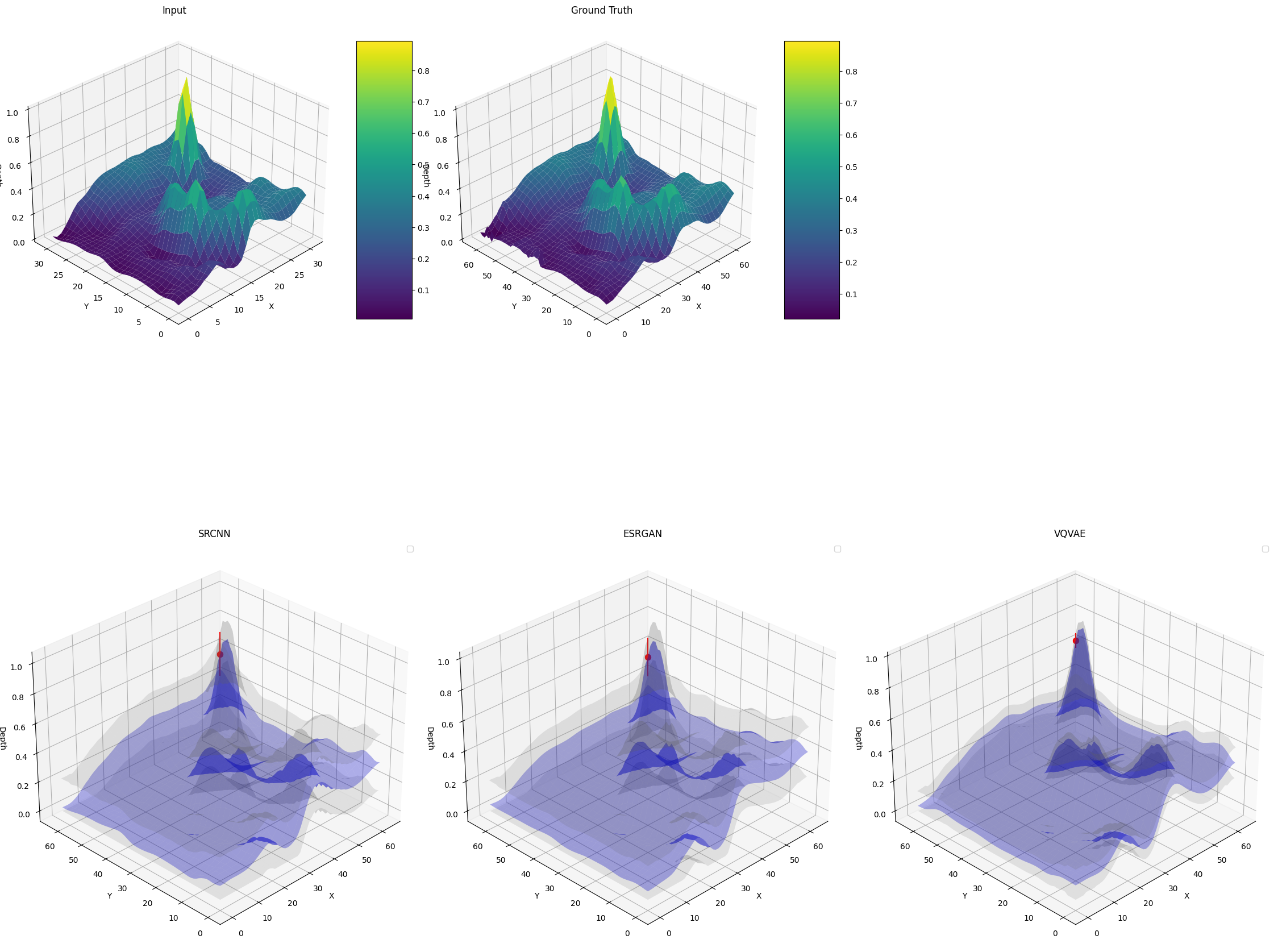}
\end{center}

\caption{Uncertainty Comparison of Models (3D)}
\end{figure}

The figure above illustrates a comparative analysis of bathymetry data, emphasizing the performance of various machine learning models (SRCNN, GAN, and VQVAE) in relation to ground truth data and an initial low-resolution input. The upper row displays the input (low-resolution) bathymetry on the left and the corresponding ground truth (high-resolution) on the right. This section of the figure serves as the baseline reference for comparison. The color bars associated with each plot indicate the respective depth values. The bottom row displays the predicted bathymetry generated by the three models. Each model aims to reconstruct high-resolution bathymetry from low-resolution inputs. The semi-transparent gray regions surrounding the predicted surfaces represent the uncertainty associated with each model's prediction; a more expansive area signifies greater uncertainty. Each model prediction includes a red line indicating the uncertainty range at a specific point on the predicted surface. This line defines the lower and upper limits of the model's confidence interval at that position, facilitating a direct visual evaluation of the extent of uncertainty. The VQVAE model, depicted in the rightmost figure of the bottom row, produces a predicted bathymetry surface alongside its corresponding uncertainty.  The comparison of the VQVAE's predictions with the ground truth demonstrates its precision in bathymetric reconstruction.  The uncertainty limits and the red line at the specified point illustrate the model's confidence in its estimations.  The uncertainty region size and the length of the red line visually illustrate the extent of uncertainty in the VQVAE output.

\newpage
\begin{figure}[h]
\begin{center}
    %\framebox[4.0in]{$\;$}
    %\fbox{\rule[-.5cm]{0cm}{4cm} \rule[-.5cm]{4cm}{0cm}}
    \includegraphics[width=0.95\linewidth]{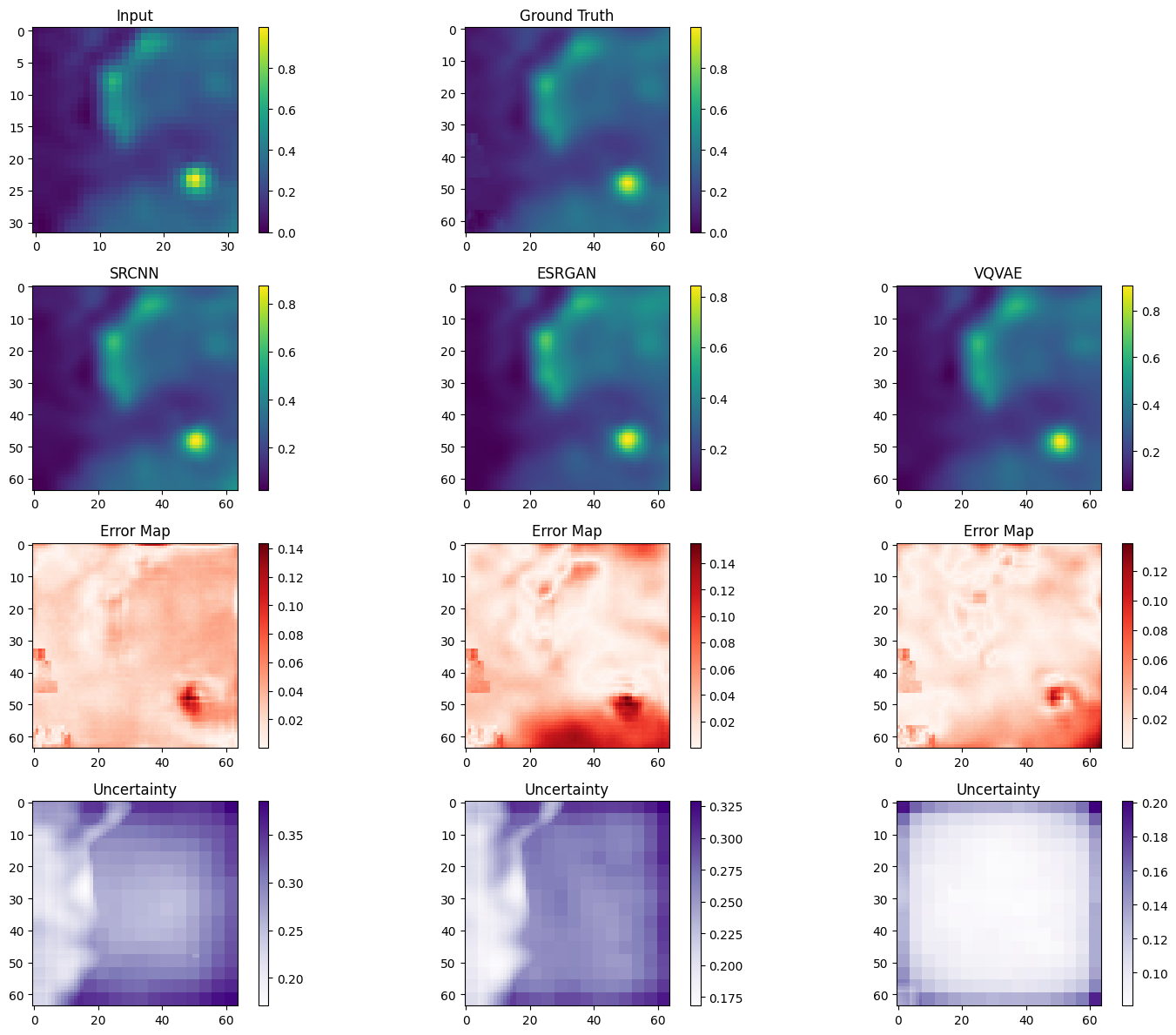}
\end{center}
\label{fig:2d-image-comparison}
\caption{2D Comparison of Models based on Error Map and Uncertainty Blocks}
\end{figure}

Meanwhile,  as shown in the 2D comparison plots, VQ-VAE exhibits superior performance in refined bathymetry generation compared with ESRGAN and SRCNN when integrated with uncertainty awareness. This can be attributed to incorporating uncertainty weights within the VQ-VAE framework. In contrast, ESRGAN's adversarial training objective, which prioritizes generating outputs that are visually indistinguishable from ground truth, can lead to suboptimal solutions in the context of precise bathymetric data reconstruction. VQ-VAE's inherent regularization encourages more robust and accurate representations, resulting in improved fidelity in the generated bathymetry.

These figures provides a qualitative evaluation of the models' abilities to accurately and their confidence with respect to reconstructing high-resolution bathymetry, highlighting predictive performance and associated uncertainty.  The visual comparison of predicted surfaces with ground truth, coupled with the analysis of uncertainty bounds—particularly for the VQVAE model—enables an evaluation of each model's strengths and weaknesses.

\newpage
\subsubsection{Uncertainty Plots of each Individual Model}

\begin{figure}[h]
\begin{center}
    \includegraphics[width=\linewidth]{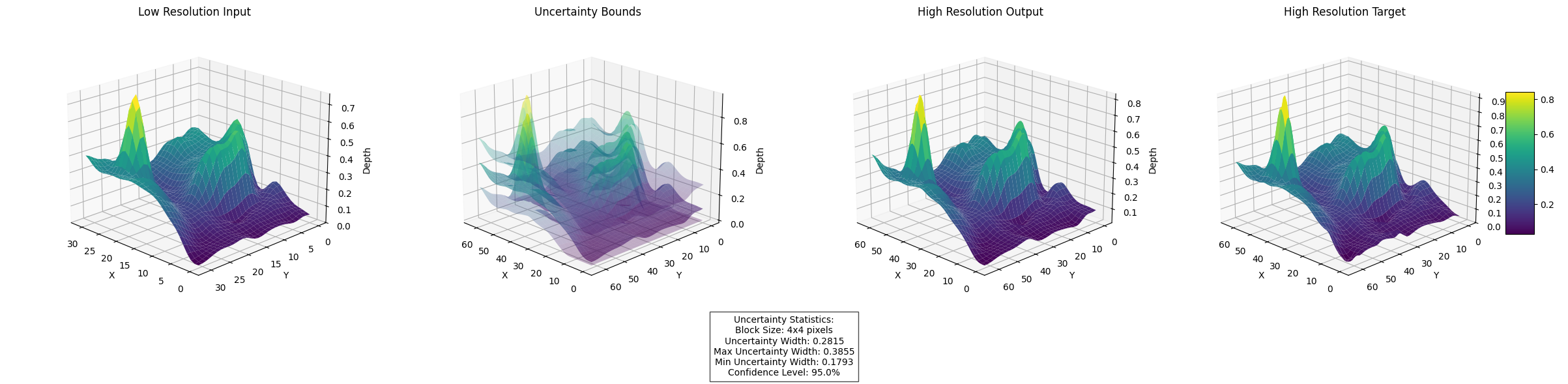}
\end{center}

\caption{High Resolution Bathymetry produced by Uncertainty-Aware SRCNN.  This visualization comprises two rows of plots. The upper row presents the progression from low-resolution input ($32\times32$) to high-resolution prediction ($64\times64$), culminating in a prediction with uncertainty bounds visualized as translucent gray bands encompassing the predicted surface. The lower row provides the ground truth high-resolution terrain alongside a spatial map of uncertainty widths, where the scale ($0.1793$-$0.3855$) indicates the magnitude of uncertainty for each block in the prediction space.}
    \label{fig:srcnn_uncertainty}
\end{figure}

\begin{figure}[h]
\begin{center}
    \includegraphics[width=\linewidth]{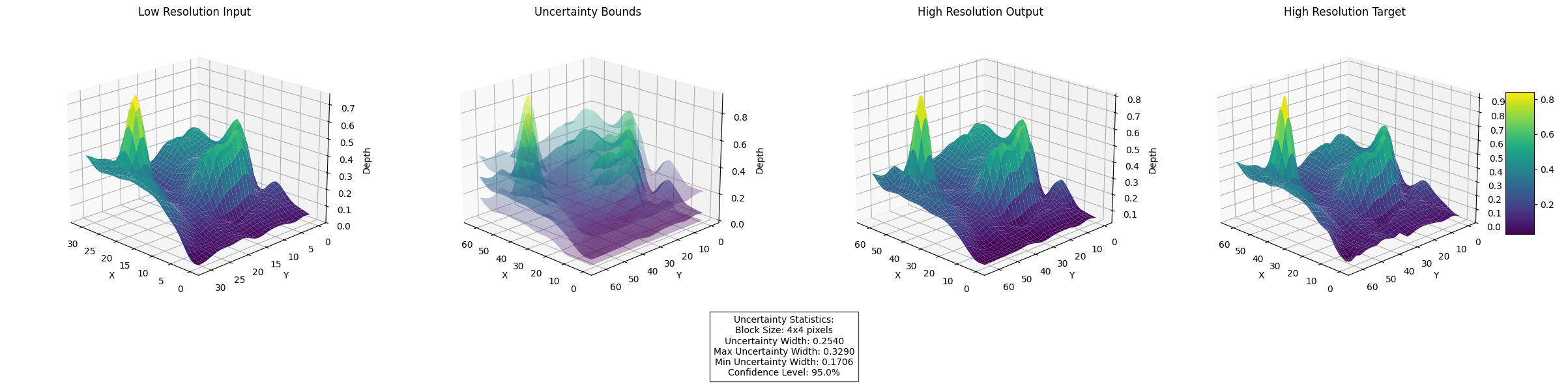}
\end{center}
\caption{High Resolution Bathymetry produced by Uncertainty-Aware ESRGAN. Following the same layout convention, the upper row shows the pipeline from input through prediction, with the rightmost plot displaying the prediction bounds. The lower row shows comparison of the ground truth with a detailed block-wise uncertainty width distribution map. The uncertainty values, ranging from $0.1706$ to $0.3290$, demonstrate the model's confidence across different spatial regions, with particular attention to areas of complex topographical features.}
    \label{fig:esrgan_uncertainty}
\end{figure}

\begin{figure}[h]
\begin{center}
    \includegraphics[width=\linewidth]{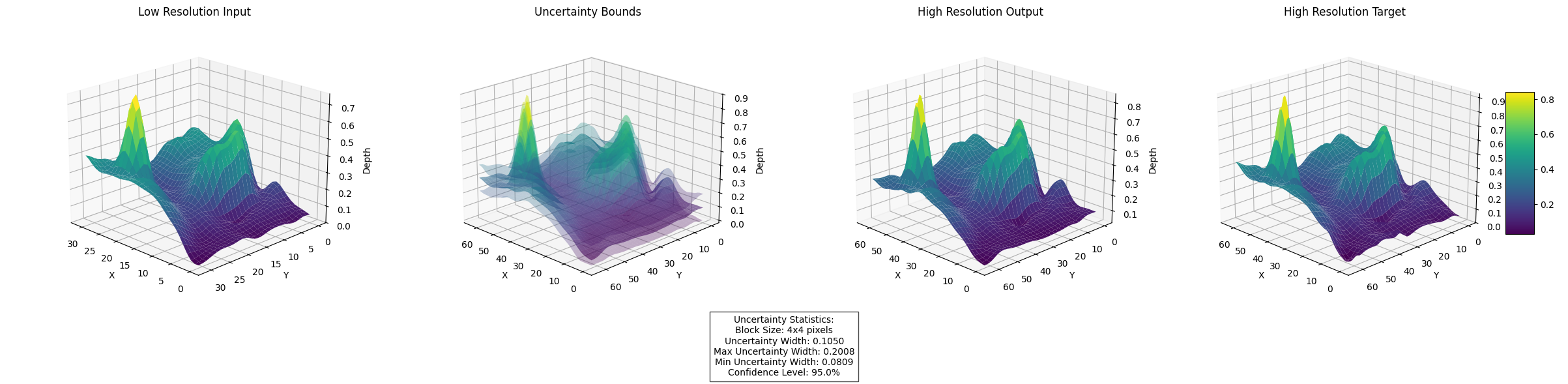}
\end{center}
\caption{High Resolution Bathymetry produced by Uncertainty-Aware VQVAE. The upper row depicts the reconstruction process, from the initial low-resolution input to the final prediction with uncertainty bounds. The lower panels contrast the ground truth terrain with a spatially-resolved uncertainty width map. Notably, the uncertainty values ($0.0809$-$0.2008$) indicate generally higher confidence levels compared to the SRCNN and ESRGAN implementations, while maintaining appropriate uncertainty estimates in regions of high complexity.}
    \label{fig:vqvae_uncertainty}
\end{figure}

\end{document}